\documentclass[conference]{IEEEtran}
\usepackage{times}

\usepackage[numbers]{natbib}
\usepackage{multicol}
\usepackage[bookmarks=true]{hyperref}

\usepackage{amsmath,amsfonts}
\usepackage{amsthm}
\usepackage{mathtools}
\usepackage{multirow}
\usepackage{float}
\usepackage{xcolor}
\usepackage{algorithm2e}
\usepackage{booktabs}

\newtheorem{theorem}{Theorem}

\newtheorem{definition}{Definition}

\newcommand{\bfx}{\mathbf{x}}
\newcommand{\bfw}{\mathbf{w}}
\newcommand{\bfu}{\mathbf{u}}

\newcommand{\R}{\mathbb{R}}

\newcommand{\E}{\mathbb{E}}

\newcommand{\cX}{\mathcal{X}}
\newcommand{\cU}{\mathcal{U}}
\newcommand{\cT}{\mathcal{T}}

\newcommand{\cE}{\mathcal{E}}
\newcommand{\cH}{\mathcal{H}}
\newcommand{\cO}{\mathcal{O}}
\newcommand{\cC}{\mathcal{C}}
\newcommand{\cF}{\mathcal{F}}
\newcommand{\cL}{\mathcal{L}}
\newcommand{\cW}{\mathcal{W}}
\newcommand{\cQ}{\mathcal{Q}}
\newcommand{\cD}{\mathcal{D}}

\newcommand{\Paths}{\mathsf{Paths}}
\newcommand{\KL}{\text{KL}}
\newcommand{\KSD}{\text{KSD}}

\newcommand{\revision}[1]{{#1}}

\DeclareMathOperator*{\argmin}{arg\,min}
\DeclareMathOperator*{\argmax}{arg\,max}
\newcommand{\secref}[1]{\S\ref{#1}}

\begin{document}

\title{Stein Variational Ergodic Search}


\author{\authorblockN{Darrick Lee}
\authorblockA{Mathematical Institute\\
University of Oxford\\
darrick.lee@maths.ox.ac.uk}
\and
\authorblockN{Cameron Lerch}
\authorblockA{Yale University\\
cameron.lerch@yale.edu}
\and
\authorblockN{Fabio Ramos}
\authorblockA{NVIDIA, USA\\
The University of Sydney, Australia\\
fabio.ramos@sydney.edu.au}
\and
\authorblockN{Ian Abraham}
\authorblockA{Yale University\\
ian.abraham@yale.edu}
}


%

\maketitle

\begin{abstract}
    Exploration requires that robots reason about numerous ways to cover a space in response to dynamically changing conditions. 
    However, in continuous domains there are potentially infinitely many options for robots to explore which can prove computationally challenging.
    How then should a robot efficiently optimize and choose exploration strategies to adopt?
    In this work, we explore this question through the use of variational inference to efficiently solve for distributions of coverage trajectories. 
    Our approach leverages ergodic search methods to optimize coverage trajectories in continuous time and space.  
    In order to reason about distributions of trajectories, we formulate ergodic search as a probabilistic inference problem. 
    We propose to leverage Stein variational methods to approximate a posterior distribution over ergodic trajectories through parallel computation. 
    As a result, it becomes possible to efficiently optimize distributions of feasible coverage trajectories for which robots can adapt exploration.
    We demonstrate that the proposed Stein variational ergodic search approach facilitates efficient identification of multiple coverage strategies and show online adaptation in a model-predictive control formulation.
    Simulated and physical experiments demonstrate adaptability and diversity in exploration strategies online. 
\end{abstract}

\IEEEpeerreviewmaketitle

\section{Introduction}

    Effective robotic exploration requires robots to reason about different ways to explore a space. 
    In the presence of uncertainty and in unstructured environments, having multiple exploration strategies for robots to quickly choose from can be advantageous. 
    However, optimizing for multiple exploratory paths can be computationally challenging, especially in continuous domains where trajectory solutions are infinite dimensional. 
    Defining the problem of exploration on a grid~\cite{choset2000coverage, bahnemann2021revisiting, galceran2013survey, sung2023survey} can provide the means to enumerate many coverage paths; however, grid-based methods limit where the robot can visit in continuous, unstructured environments. 

    Recent advancements in curiosity- and information- based exploration have allowed robots to explore vast domains~\cite{sung2023survey, pathak2017curiosity, mazzaglia2022curiosity, silverman2013optimal, Chen-RSS-22}. 
    However, they are often limited to exploring using one strategy, i.e., an information maximizing strategy. 
    As a result, exploring becomes myopic where immediate information gain is sought after without regard to advantageous states in the future.
    Ergodicity-based exploration techniques show promise in breaking away from myopic strategies by posing exploration as a coverage problem based on time-averaged trajectory visitation~\cite{mathew2011metrics, miller2013trajectory}. 
    More specifically, ergodic methods optimize over where trajectories spend time on average as a function of the expected measure of information. 
    As a result, ergodic search methods allow robots to optimize exploratory paths in multi-modal search problems~\cite{miller2015ergodic}. 
    While the literature has proven ergodic-methods to produce effective exploration strategies, they only optimize one search strategy at any given moment, thus limiting how robots can adapt. 

    \begin{figure}
        \centering
        \includegraphics[width=\linewidth]{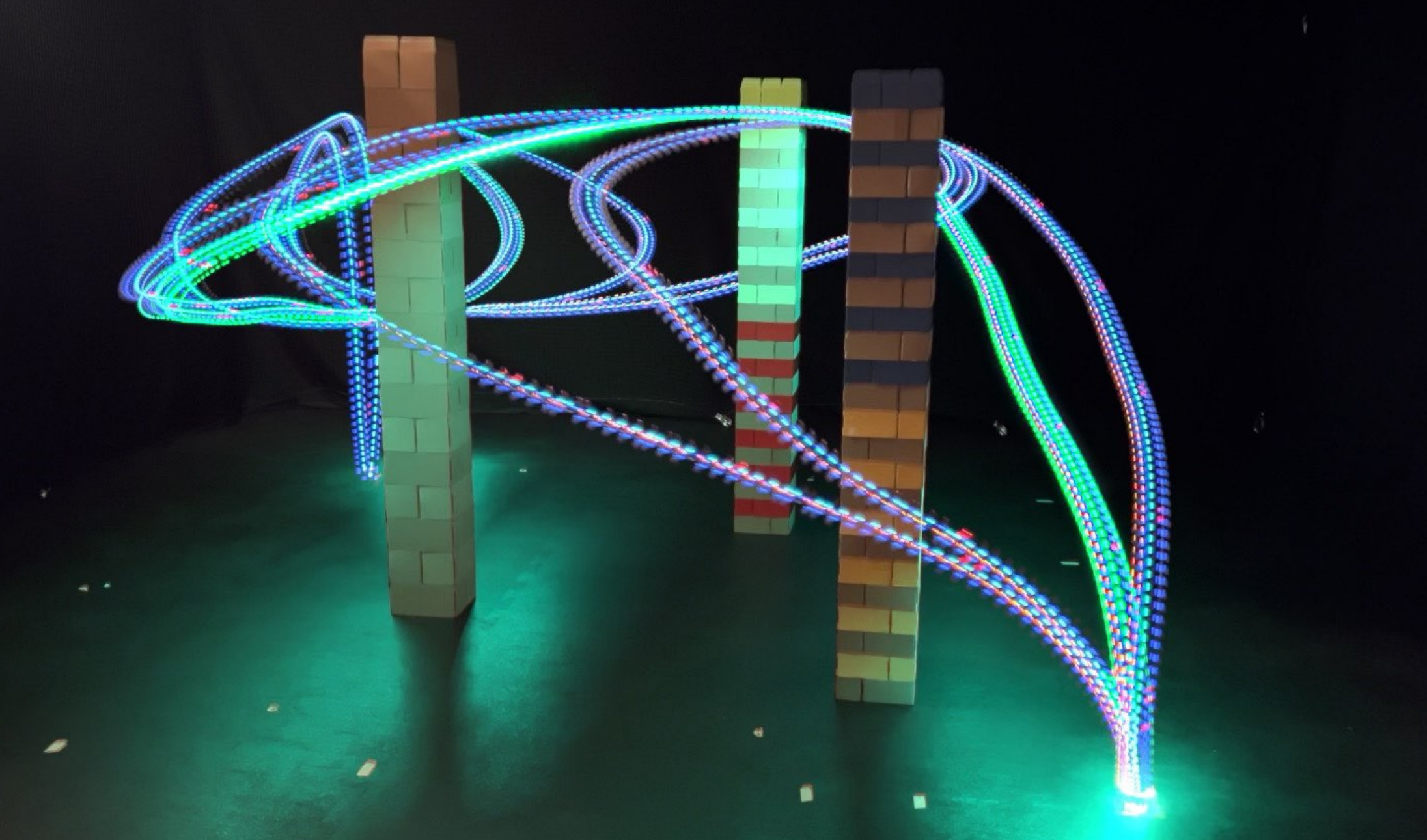}
        \caption{\textbf{The Stein Variational Ergodic Approach. } Robotic exploration is challenging as there are many ways to effectively explore an area that robots need to reason about. Calculating all the possible exploration strategies can be computationally prohibitive, especially when there is no guarantee that optimized solutions will coverage to a diverse set of strategies. We propose to solve this problem by posing coverage and exploration as an inference problem over distributions of trajectories. Our approach leverages ergodic exploration techniques in conjunction with Stein variational methods to efficiently optimize diverse exploration strategies in parallel over continuous domains (see above). Illustrated is a set of $4$ exploration strategies optimized to uniformly explore around the cylinders. The green trajectory indicates the selected best strategy.}
        \label{fig:enter-label}
    \end{figure}

    Prior work suggests that the non-convex form of ergodic search methods are capable of identifying multiple optimal trajectories through variations in initial condition~\cite{miller2013trajectory}. 
    The challenge with reasoning about multiple, i.e., distributions, of trajectories is that the optimization can become computationally prohibitive. 
    Furthermore, there is no guarantee that certain initial conditions will not converge onto the same ergodic search strategy which is not ideal, especially in online exploration scenarios where mode collapse on trajectory solutions can have catastrophic consequences. 

    Stein variational inference methods show promise in providing the necessary tools to approximate distributions of trajectories in a computationally tractable manner~\cite{liu2016stein}.
    These methods leverage approximate inference in a non-parametric manner that 1) empirically estimates complex distributions, and 2) can do so in a computationally efficient manner through parallel computation. 
    Thus, in this work, we propose a novel formulation of ergodic exploration using Stein variational gradient decent methods~\cite{liu2016stein} that solves an inference problem on distributions of trajectories.
    We find that the spectral construction of the ergodic metric promotes discovery of locally optimal solutions which can be leveraged to guide robot exploration with multiple redundancies. 
    We demonstrate the efficacy of our approach in simulation and on a physical drone system that can efficiently and effectively adapt exploration in cluttered environments.
    Furthermore, we find that the synergy between Stein variational methods and ergodic search methods promotes a diverse range of exploration strategies which can be acquired in a computationally efficient manner. 
    In summary, our contributions are 
    \begin{enumerate}
        \item A Stein variational ergodic search method for trajectory optimization and control; 
        \item Demonstration of diverse exploration strategies in a computationally efficient manner; and
        \item Real-time exploration and adaptation of multiple ergodic trajectories in dynamic domains.
    \end{enumerate}

    The remainder of the paper is structured as follows: Section~\secref{sec:related_work} provides an overview of related work, Section~\secref{sec:ergodic_search} introduces the problem of coverage via ergodic search, Section~\secref{sec:stein_var} introduces the Stein variational gradient descent, Section~\secref{sec:stein_var_erg} derives the proposed Stein variational ergodic search approach, and Sections~\secref{sec:results}, and~\secref{sec:conclusion} presents the results and conclusions. 
    
\section{Related Work} \label{sec:related_work}

    \textbf{Coverage and Exploration.}
        Robotic exploration seeks to guide robots towards unexplored areas within a domain.
        Likewise, coverage is concerned with generating paths and placement of a robot's sensors such that it covers a bounded domain~\cite{butler1999contact, acar2002morse}. 
        Early solutions to coverage and exploration are often formulated over discrete grids (defined on continuous space) and shown to have completeness guarantees through novel boustrophedon search patterns (which originates the lawnmower pattern~\cite{choset2000coverage}) and the traveling salesperson problem~\cite{ applegate2011traveling, laporte1983generalized}.
        Through grids, it is relatively straightforward to use exhaustive computational techniques to enumerate the many different, feasible paths a robot can take to explore and cover an area. 
        Extending these ideas to continuous domains (which provides infinite spatial resolution for the robot to traverse) often proves challenging as the number of possible ways to explore a domain becomes infinite. 

        Recent methods in reinforcement learning and information-based methods circumvent the issues with transitioning to continuous domains by leveraging ``curiosity'' measures~\cite{pathak2017curiosity, burda2018large} derived from information theory~\cite{chirikjian2011stochastic} to encourage exploration.
        These measures provide a signal which informs a robot where it is beneficial to visit and have demonstrated comparability with continuous domains~\cite{pathak2017curiosity, mazzaglia2022curiosity}.
        However, many of these methods tend to be myopic in nature, only focusing on immediate information gain, and often limited to single-mode solutions which can limit the adaptability of robots in the wild. 
        More recent efforts in ergodicity-based methods (also referred to as ergodic exploration, coverage, or search) have demonstrated that it is possible to compute intricate coverage patterns through a spectral-based metric over continuous domains~\cite{miller2015ergodic, abraham2020active, Dong-RSS-23}. 
        Ergodicity-based methods generate coverage by minimizing the difference between the (expected) spatial distribution of information and the time-averaged spatial distribution of an agent's trajectory within the domain \cite{mathew2011metrics}. 
        Interestingly, prior work has provided evidence that suggests ergodic coverage strategies produce optimal exploration strategies~\cite{dressel2018optimality}. 
        The spectral, multi-scale composition of the ergodic metric~\cite{scott2009capturing} suggests there exist many solutions that can be exploited by robots.
        However, it has yet to be demonstrated how one can calculate sets of ``good'', locally-optimal ergodic trajectory solutions that a robot can adapt as its search strategy. 

    \textbf{Control as Inference.} 
        The introduction of probabilistic inference to optimal control has enhanced the capabilities of robots via sample-based methods. 
        Specifically, the non-convexity of many robot task specifications makes it challenging for gradient-based methods to find reasonable trajectory solutions that satisfy a task. 
        Incidentally, this is caused by the non-convexity of many robotic tasks due to its specification or the underlying complexity of the interactions, e.g., contact dynamics. 
        Sample-based methods derived from probabilistic inference circumvent these issues through zero-order optimization, e.g., predictive sampling~\cite{howell2022predictive}, and model-predictive path integral (MPPI)~\cite{williams2017model, williams2016aggressive}, that are less sensitive to the non-convexity of robotic tasks. 
        Add in the significant technological leap of GPU-based computation, and control as an inference problem becomes a powerful tool for robotics. 
        However, like with gradient-based methods, many of the sample-based techniques typically solve for only one locally optimal solution, ignoring other equally viable ways to solve a task. 
        
        Having more than one plan for which robots can switch between is highly valuable for real-world systems that need to quickly change strategy. 
        For example, in crowd-based navigation, certain planned paths often become infeasible due to the dynamically changing environment. 
        Rather than having to recalculate a new path, it is more efficient and robust for a robot to have several paths to choose from~\cite{SunM-RSS-21, lambert2021stein}. 
        The recent adoption of optimal transport and variational methods in optimal control has demonstrated promise is isolating the many locally optimal solutions for robots to choose from~\cite{le2023accelerating, lambert2021stein}. 
        These approaches pose trajectory solutions to non-convex optimal control problems as inference problems were solutions are approximated as distributions. 
        As a result, it becomes possible for robots to optimize for many locally optimal solutions and adapt in real-time. 
        However, most of the applications have focused on obstacle avoidance and point-to-point navigation. 
        Exploration and coverage has not yet been explored primarily due to 1) the difficulty of forming coverage objectives in continuous domains, and 2) stable convergence of solutions over long-time horizons. 
        
        In this work, we demonstrate it is possible to solve for many coverage trajectory solutions in continuous space that promote diverse and robust exploration over long-time horizons by forming ergodic coverage methods as an inference problem and leverage second-order Stein variational gradients~\cite{lambert2021stein, liu2016stein} to converge on solutions.

\section{Background and Preliminaries}

\subsection{Ergodic Search} \label{sec:ergodic_search}

    Let us first define a robot's state space and control space as $\cX \subseteq \R^n$ and $\cU \subseteq \R^m$. 
    Next let us define the robot's state trajectory $x(t) : \R^+ \to \mathcal{X}$ as the solution to the initial value problem 
    \begin{equation} \label{eq:dynamics}
        x(t) = x(0) + \int_{0}^t f(x(\tau), u(\tau)) d\tau
    \end{equation}
    where $x(0) \in \cX$ is an initial condition, $u(t) : \R^+ \to \cU$ is a control trajectory, and $f(x,u) : \cX \times \cU \to \cT_\cX$ is the continuous-time (potentially nonlinear) dynamics of the robot.
    We denote
    \begin{equation} \label{eq:paths}
        \Paths(\cX) \coloneqq \bigcup_{T > 0} \cC([0,T], \cX)
    \end{equation}
    as the set of all continuous trajectories defined on arbitrary finite time intervals.
    In addition, let us define a bounded domain where the robot explores as $\mathcal{W}= [0,L_0]\times \ldots \times [0, L_{v-1}]$, where $v\leq n$ and $L_i$ are the bounds of the workspace. 
    We fix a map $g(x) : \mathcal{X} \to \mathcal{W}$ that projects state space $\mathcal{X}$ to exploration space $\mathcal{W}$, e.g., a selection matrix $g(x) = \mathbf{S}x$ that isolates and scales certain states that correspond to exploration.
    \begin{definition} \textbf{Time-Averaged Trajectory Statistics.} \label{def:ergodicity}
        Let $m$ denote the Lebesgue measure on $\R^+$, $x(t): \R^+ \to \cX$ be a trajectory, and $g: \cX \to \cW$.
        For each $T \in \R^+$, let the probability measure $\cQ_T$ on $\cW$ that defines the time-averaged trajectory visitation statistics integrated along time $[0,T]$ be defined by
        \begin{equation} \label{eq:Q_T}
            \cQ_T(A) \coloneqq \frac{1}{T} m\big( (g\circ x)^{-1}(A) \cap [0,T]\big),
        \end{equation}
        where $A \subset \cW$ is a Borel set.
    \end{definition}

    \begin{definition} \textbf{Ergodicity.} A trajectory $x(t)$ is ergodic with respect to a Borel probability measure $\mu$ on $\cW$ if $\cQ_T$ converges weakly to $\mu$ as $T \to \infty$. That is, 
        \begin{equation}
            \lim_{T \to \infty} \int_{\cW} \phi(w)\, d\cQ_T(w) = \int_{\cW} \phi(w)\, d\mu(w)
        \end{equation}
        for all continuous functions $\phi \in \cC(\cW)$. 
        In particular, the trajectory statistics measure can be viewed as an integral of delta functions, where
        \begin{equation} \label{eq:def_of_QT_integral}
            \int_{\cW} \phi(w) \,d\cQ_T(w) = \frac{1}{T}\int_0^T \phi(g \circ x(t)))\, dt
        \end{equation}
        which is the common definition in literature~\cite{mathew2011metrics, miller2013trajectory}.
    \end{definition}
    Roughly speaking, a trajectory is ergodic with respect to the measure $\mu$ if it eventually (as $T \to \infty$) explores the workspace in a manner which is commensurate with $\mu$.
    This is formalized by requiring the measure $\cQ_T$, which quantifies the proportion of time the trajectory is contained in subsets $A \subset \cW$ during the interval $[0,T]$, to converge to $\mu$.
    Because $\cW$ is compact, weak convergence $\cQ_T \xrightarrow{T \to \infty} \mu$ is defined by
    \begin{equation}
        \lim_{T \to \infty} \int_{\cW} \phi(w)\, d\cQ_T(w) = \int_{\cW} \phi(w)\, d\mu(w)
    \end{equation}
    for all integrable functions $\phi \in C(\cW)$.
    However, since robots only run for finite time horizons, we must quantify the level of ergodicity of finite trajectories.

    \begin{definition} \textbf{Ergodic Cost Function.}
        Let $\mu$ be a probability measure on $\cW$. A \emph{$\mu$-ergodic cost function} is a function $\cE_\mu: \Paths(\cX) \to \R$ such that for an infinite trajectory $x(t): \R^+ \to \cX$, if $\cE_\mu(x|_{[0,T]}) \to 0$ as $T \to \infty$ then $x(t)$ is ergodic.
    \end{definition}
    To define the ergodic metric for trajectory optimization, we use spectral methods and construct a metric in the Fourier space~\cite{mathew2011metrics, scott2009capturing, miller2015ergodic}.
    \begin{definition} 
        \label{def:erg_metric} \textbf{Spectral Ergodic Cost Function.}
        Let $\mu$ be a probability measure on $\cW$.
        Let $\mathcal{K}^v \subset \mathbb{N}^{v}$ be the set of all integer $k$ fundamental frequencies that define the cosine Fourier basis function 
        \begin{equation} \label{eq:Fk}
            F_k(w) = \frac{1}{h_k} \prod_{i=0}^{v-1} \cos\left(\frac{w_i k_i \pi}{L_i}\right)
        \end{equation}
        where $h_k$ is a normalizing factor (see \cite{miller2015ergodic,mathew2011metrics}).
        For a finite trajectory $x(t):[0,T] \to \cX$, let $\cQ_T$ be the measure defined in Eq.~\eqref{eq:Q_T}.
        The spectral ergodic cost function is defined as
        \begin{align}\label{eq:ergodic_met}
            &\mathcal{E}_\mu(x) = \sum_{k\in \mathcal{K}^v} \Lambda_k \left( Q^k_T - \mu^k \right)^2 \\
            &= \sum_{k\in  \mathcal{K}^v} \Lambda_k \left( \frac{1}{T}\int_{0}^{T}F_k(g\circ x(t)) dt - \int_{\mathcal{W}} F_k(w)d\mu(w) \right)^2 \nonumber
        \end{align}
        where $\cQ_T^k$ and $\mu^k$ are the $k^\text{th}$ Fourier decomposition modes of $\cQ_T$ and $\mu$, respectively (using Eq.~\eqref{eq:def_of_QT_integral}), and
        $\Lambda_k=(1 + \left\Vert k \right\Vert_2)^{-\frac{v+1}{2}}$ is a weight coefficient that places higher importance on lower-frequency modes.
    \end{definition}
    In particular, the spectral ergodic cost function defined for a probability measure $\mu$ forms a metric~\cite{mathew2011metrics} which is able to generate coverage trajectories at arbitrary spacial scales~\cite{scott2009capturing} (which we additionally prove in Appendix~\ref{apx:ergodic_cost}).
    This is advantageous as 1) the metric defines a proper coverage distances using spectral modes in continuous space; and 2) specifies infinitely many ways to explore with respect to the spectral decomposition. Note that while the ergodic metric forms a non-convex objective in the trajectory space, the non-convexity specifies many trajectory solutions that robots can take leverage to improve exploration.
    The challenge is then optimizing for the distribution of feasible ergodic trajectory solutions.

    We can formulate an optimization problem over trajectories $x(t)$ through the ergodic metric. 
    Consider that $x(t)$ is constrained (either through dynamics~\eqref{eq:dynamics} or through some other differential constraints). 
    Then the ergodic trajectory optimization is given as 
    \begin{align}
        \min_{\substack{x(t) \in \Paths(\cX)  \\  u(t) \in \cU \, \forall t \in [0, T] }} & \cE_\mu(x(t)) \\ 
        \text{subject to } & h_1(x) = 0 \,\, \forall t \in [0, T] \nonumber \\ 
        & h_2(x) \le 0 \,\, \forall t \in [0,T] \nonumber
    \end{align}
    where initial conditions and input trajectory constraints $u(t)$ can be accounted for in the equality and inequality constraints $h_1, h_2$ , e.g., $\dot{x} = f(x, u)$ and $u \in \cU$.

\subsection{Stein Variational Gradient Descent} \label{sec:stein_var}
    A powerful method to perform probabilistic inference is variational inference (VI), where one aims to approximate a complex target distribution $p(x)$ with a candidate distribution $q(x)$ from a parameterized family of distributions $\cD$. This is typically achieved by minimizing the Kullback-Leibler (KL) divergence,
    \begin{align}
        q^* = \argmin_{q \in \cD} D_{\KL}(q ||p).
    \end{align}
    However, finding an appropriate family $\cD$ which balances accuracy with tractable computations is often challenging. Stein variational gradient descent (SVGD)~\cite{liu2016stein} addresses this issue by providing a nonparametric method to sample from the target distribution $p(x)$ via kernel methods and gradient descent. 

    The Stein variational method is derived by sampling a collection of points $\{x^i_0\}_{i=1}^N$ using a prior distribution on $\cX$, and then iteratively updating them by the gradient descent
    \begin{align} \label{eq:update_step}
        x^i_{r+1} = x_r^i + \epsilon \phi^*_{r}(x_r^i).
    \end{align}
    In particular, $\epsilon > 0$ is the step size, and $\phi^*_r: \cX \to \cT_\cX$ is the vector field which maximally decreases the KL-divergence at the $r^{th}$ step (not to be confused with fundamental frequency).

    Suppose $k: \cX \times \cX \to \R$ is a positive definite kernel, and $\cH$ is its corresponding reproducing kernel Hilbert space (RKHS). We restrict the vector field to be the function class $\cH^n$, and therefore
    \begin{align}
        \phi^*_r = \argmax_{\phi \in \cH^n} \{ - \nabla_\epsilon D_{\KL}(\hat{q}_r || p) \, : \, \|\phi\|_{\cH^n} \leq 1\},
    \end{align}
    where $\hat{q}_r = \frac{1}{N} \sum_{i=1}^N \delta_{x^i_{r}}$ denotes the empirical distribution at step $r$. Note that if the kernel $k$ is universal, $\cH^n$ is dense in the space of continuous functions $C(\cX, \R^n)$, and does not result in a loss of generality. Furthermore, from~\cite{liu2016stein}, we have an exact form for $\phi^*_r$,
    \begin{align} \label{eq:svgd_phi}
        \phi^*_r(\cdot) = \E_{x \sim \hat{q}_{r-1}} \left[ k(x,\cdot) \nabla_x \log p(x) + \nabla_x k(x,\cdot)\right]
    \end{align}
    which can be approximated by samples $\{x_r^i\}^N_{i=1}$
    \begin{align} \label{eq:approx_svgd_phi}
        \phi^*_r(\cdot) = \frac{1}{N} \sum_{n=1}^N k(x^i_r,\cdot) \nabla_x \log p(x^i_r) + \nabla_x k(x^i_r,\cdot).
    \end{align}
    Thus, the sampled points converge onto an approximation of the distribution $p$. 
    In this work, we leverage the Stein variational gradients to define and solve for distributions of ergodic trajectories for robot exploration.
    
\section{Ergodic Coverage as Inference} \label{sec:stein_var_erg}

    In this section, we pose the problem of computing ergodic trajectories as an inference problem over a distribution of paths.
    We then derive an algorithm for optimizing approximate distributions on ergodic paths and prove optimality of the posterior distribution of paths in the ergodic coverage inference problem.

    \subsection{The likelihood of ergodicity and Stein variations}
        In order to explore the landscape of trajectories which minimize an ergodic cost function, we will formulate the ergodic search problem in terms of variational inference.
        Adapting~\cite{lambert2021stein}, which formulated motion planning problems using variational inference, we introduce a binary optimality criterion $\cO: \Paths(\cX) \to \{0,1\}$. 
        Suppose $\mu$ is a probability measure on $\cW$.
        Then let use this optimality to encode a $\mu$-ergodic cost likelihood function $\cE_\mu$ by defining
        \begin{align}
            p(\cO | x) \coloneqq \exp(-\lambda \cE_\mu(x)),
        \end{align}
        where $\lambda > 0$ is a hyperparameter and we simplify the notation by using $\cO$ to denote the optimal condition $\{\cO = 1\}$. 
        Given a prior distribution $p$ on $\Paths(\cX)$, and a positive definite kernel $k: \Paths(\cX)^2 \to \R$ on path space, we can use SVGD to sample from the posterior distribution $p(x|\cO)$. In particular, by applying Bayes' rule, the optimal vector field from Eq.~\eqref{eq:svgd_phi} for a $\mu$-ergodic cost likelihood function $\cE_\mu$ is given by 
        \begin{multline} \label{eq:phi_star}
            \phi^*_r 
            = \E_{x \sim \hat{q}_{r-1}} [ k(x,\cdot) (\nabla_x \log p(x) - \\ \lambda \nabla_x \cE_\mu(x)) + \nabla_x k(x,\cdot)],
        \end{multline}
        where $x \in \Paths(\mathcal{X})$ and $p(x)$ is a prior over trajectories. 
        The main benefit of forming an ergodic trajectory optimization through Stein variational descent is the ability to optimize multiple trajectories in parallel. 
        
    \subsection{Ergodic Stein variational trajectory optimization}
        
        Instead of defining a distribution of infinite-dimensional trajectories, we discretize the path to facilitate numerical optimization. 
        Let $\bfx = [x_0, x_1, \ldots, x_{T-1}]$, where $x_t \in \mathcal{X}$, be a collection of points representing a trajectory along a discrete time horizon $T$  indexed by discrete time $t$ (which we purposefully overload the notation to keep consistent with the stated definitions). 
        Note that the definition of paths in Eq.~\eqref{eq:paths}, ergodicity in Def.~\ref{def:ergodicity}, and as a cost function Def.~\ref{def:erg_metric} are now calculated on discrete paths and still hold. 
        Next, let us define an empirical distribution over $N$ discrete paths $\bfx$ as $\hat{q} = \frac{1}{N} \sum_{i=1}^N \delta_{\bfx^i}$.
        Following Eq.~\eqref{eq:approx_svgd_phi}, we approximate $p(\bfx)$ using $\hat{q}$ where we assign $p(\bfx) = \mathcal{N}(\hat{\bfx}, \sigma^2)$ where $\hat{\bfx}=\text{interp}(x_\text{init}, x_\text{final})$, and $\sigma^2$ is the variance. 
        Given a kernel function $k(\bfx, \cdot)$ on discrete paths $\cX^T$, the ergodic Stein variational step is given by  
        \begin{multline} \label{eq:steinE_step}
            \phi^*_r(\cdot)
            = \frac{1}{N} \sum_{i=1}^N [ k(\bfx^i_r,\cdot) (\nabla_\bfx \log p(\bfx_r^i) - \\ \lambda \nabla_\bfx \cE_\mu(\bfx_r^i)) + \nabla_\bfx k(\bfx_r^i,\cdot)]
        \end{multline}
        where $-k(\bfx^i_r,\cdot)\lambda \nabla_\bfx \cE_\mu(\bfx_r^i)$ minimizes the ergodic cost over the trajectory sample, and $k(\bfx_r^i,\cdot)$ is a repulsive force that pushes trajectory solutions. 
        We can view the kernel as a similarity measure between two trajectories $\bfx, \bfx^\prime$.
        Specifically, as $\bfx \to \bfx^\prime$ so does $k(\bfx, \bfx^\prime) \to 1$ and  as $\Vert \bfx-\bfx^\prime \Vert \to \infty$, $k(\bfx, \bfx^\prime) \to 0$ (when choosing a radial basis kernel~\eqref{eq:rbf}).
        Thus, one can measure diversity of a set of particles $\{ \bfx^i \}$ using a kernel by forming a matrix with entries $K_{ij}=k(\bfx^i, \bfx^j)$, (where $K_{ii} = 1$) and computing the determinant of $K$ we can establish a measure of diversity. In particular, given a set of trajectories $\{ \bfx^i \}$, $\det(K) \to 0$ as $\bfx^i \to \bfx^j \, \, \forall i,j ,\,\, i\neq j$ and $\det(K) \to 1$ as $\Vert \bfx^i -\bfx^j \Vert \to \infty \, \, \forall i,j ,\,\, i\neq j$. 
        
        Constraints on the trajectory and additional terms that shape trajectory solutions can be added by extending the cost likelihood function in the following form 
        \begin{equation} \label{eq:likelihood}
            \cL_\mu(\bfx) = \cE_\mu(\bfx) + \rho(\bfx) + c_1 h_1(\bfx)^2 + c_2 \text{max}(0,h_2(\bfx))
        \end{equation}
        where $c_1, c_2$ are positive penalty weights\footnote{$c_1, c_2$ are chosen arbitrarily, but can be treated as Lagrange multipliers in an Augmented Lagrange formulation.} that form an inner product with equality and inequality functions $h_1, h_2$, and $\rho : \Paths(\cX) \to \mathbb{R}$ is any additional penalty terms on the trajectory $\bfx$. 
        We can rewrite the ergodic Stein variational step using $\cL_\mu$ as
        \begin{multline} \label{eq:steinE_step}
            \phi^*_r(\cdot)
            = \frac{1}{N} \sum_{i=1}^N [ k(\bfx^i_r,\cdot) (\nabla_\bfx \log p(\bfx_r^i) - \\ \lambda \nabla_\bfx \cL_\mu(\bfx_r^i)) + \nabla_\bfx k(\bfx_r^i,\cdot)].
        \end{multline}
        Selection of trajectories can be done through a heuristic $\argmax$ operation 
        \begin{equation} \label{eq:heuristic}
            i^\star = \argmax_i \exp(-\lambda \cL_\mu(\bfx_r^i)).
        \end{equation}
        In Alg.~\ref{alg:stein_erg_traj_opt} we outline the Stein variational ergodic trajectory optimization algorithm.

        \RestyleAlgo{ruled}
        \SetKwComment{Comment}{/*}{*/}
        \begin{algorithm}[t!]
        \caption{Stein Variational Ergodic Trajectory Opt.}\label{alg:stein_erg_traj_opt}
        \textbf{input:} {measure $\mu$, domain $\cW$, map $g: \cX \to \cW$, cost $\cL_\mu$, prior $p(\bfx)$, kernel $k(\bfx, \cdot)$, step size $\epsilon$, iteration $r=0$, initial trajectory samples $\{\bfx_0^i\}^N_{i=1}$, termination condition $\gamma$} \\ 
        \While{$\Vert \phi^\star_r(\bfx_r^i) \forall i \Vert \ge \gamma$ or $r<$ max iterations}{
            \For{ each sample $i$ in parallel}{
                $\bfx_{r+1}^i \gets \bfx_r^i + \epsilon \phi^\star_{r}(\bfx_r^i)$\;
            }
            $r \gets r + 1$
        }
        \textbf{return:}  $\{ \bfx_r^i\}_{i=1}^N$, $\argmax_i \exp(-\lambda \cL_\mu( \bfx_r^i))$
        \end{algorithm}

    \subsection{Ergodic Stein variational control}
        The ergodic Stein step Eq.~\eqref{eq:steinE_step} acts on discrete points that represent robot trajectories. 
        We can readily extend this formulation over discrete control inputs $\bfu = [u_0, u_1, \ldots, u_{T-1}]$. 
        Consider the discrete-time transition dynamics 
        \begin{equation}\label{eq:disc_dyn}
            x_{t+1} = F(x_t, u_t),
        \end{equation}
        then, given an initial condition $x_0$, and a sequence of controls $\bfu$, $\bfx$ is calculated through recursive application of Eq.~\eqref{eq:disc_dyn} starting from $x_0$. 
        Using this formulation, we can define the problem over control inputs $\bfu$ where we optimize over $p(\bfu | \cO)$ and by an approximate distribution over control sequences $\hat{q} = \frac{1}{N} \sum_{i=1}^N \delta_{\bfu^i}$.
        We can then optimize strictly over controls by using a kernel $k$ on discrete control paths $\cU^T$, where the Stein variational step is
        \begin{multline} \label{eq:steinE_step_control}
            \phi^*_r(\cdot)
            = \frac{1}{N} \sum_{i=1}^N [ k(\bfu^i_r,\cdot) (\nabla_\bfu \log p(\bfu_r^i) - \\ \lambda \nabla_\bfu \cE_\mu(\bfx_r^i | \bfu_r^i, x_0)) + \nabla_\bfu k(\bfu_r^i,\cdot)]
        \end{multline}
        where $\cE_\mu(\bfx_r^i | \bfu_r^i, x_0)$ is the ergodic metric of trajectory $\bfx$ given initial condition $x_0$ and control sequence $\bfu$.
        This form is useful in adaptive model-predictive control (MPC) where direct control values are replanned online.
        As with trajectory optimization, we can rewrite the Stein step using Eq.~\eqref{eq:likelihood} and introduce constraints on the control and trajectory as needed. 
        In most MPC formulations, after the control trajectory is optimized, it is common to pass to the robot the first control value and shift the sequence of controls, i.e., $u_{0:T-2} = u_{1:T-1}$ to warm start the optimization. 
        We follow~\cite{liu2016stein} for updating the controls. 
        We outline the Stein variational ergodic control in Alg.~\ref{alg:stein_erg_control}.
        
        \RestyleAlgo{ruled}
        \SetKwComment{Comment}{/*}{*/}
        \begin{algorithm}[t!]
        \caption{Stein Variational Ergodic Control}\label{alg:stein_erg_control}
        \textbf{input:} {initial state $x_0$, time horizon $T$, measure $\mu$, domain $\cW$, map $g: \cX \to \cW$, cost $\cL_\mu$, prior $p(\bfu)$, kernel $k(\bfu, \cdot)$, step size $\epsilon$, prior control samples $\{\bfu_0^i\}^N_{i=1}$, termination condition $\gamma$} \\ 
        iteration $r=0$\;
        \While{$\Vert \phi^\star_r(\bfu_r^i) \forall i \Vert \ge \gamma$ or $r<$ max iterations}{
            \For{ each sample $i$ in parallel}{
                $\bfu_{r+1}^i \gets \bfu_r^i + \epsilon \phi^\star_{r}(\bfu_r^i)$\;
            }
            $r \gets r + 1$
        }
        \textbf{return:}  $\{ \bfu_r^i\}_{i=1}^N$, $i^\star = \argmax_i \exp(-\lambda \cL_\mu( \bfu_r^i))$ \;
        apply $u_0^{i^\star}$ to robot\; 
        \Comment*[l]{shift controls}
            \For{ each sample $i$ in parallel}{
                $u_{0:T-2}^i \gets u_{1:T-1}^i$\;
            }
        \Comment*[l]{sample state and return to input}
        \end{algorithm}

   \begin{figure}
        \centering
        \includegraphics[width=\linewidth]{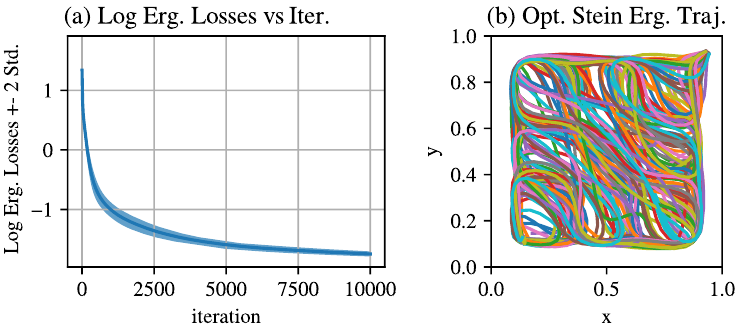}
        \caption{\textbf{Stein Ergodic Solutions Converge to Similar Ergodicity.} (a) Log of mean ergodic losses for $50$ trajectories optimized by the Stein variational ergodic search on a uniform $\mu$. Note the tight bound of 2-standard deviation on the ergodic losses suggests trajectory solutions are close in optimality. (b) Overlap of $50$ ergodic trajectories that provide uniform coverage over the $1m \times 1m$ domain.}
        \label{fig:convergence}
    \end{figure}
       \begin{figure}
            \centering
            \includegraphics[width=\linewidth]{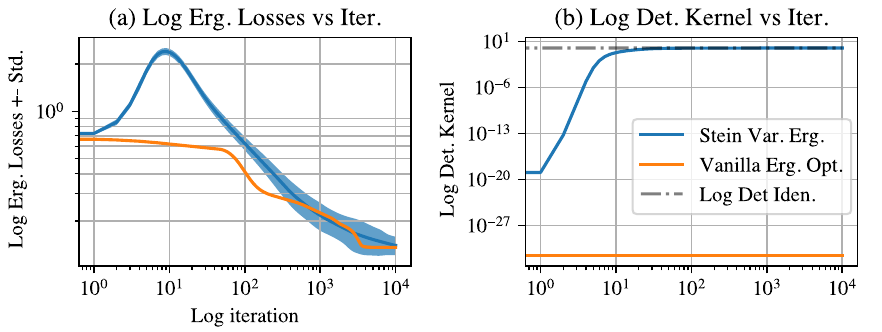}
            \caption{\textbf{Measured Kernel Ergodic Trajectory Diversity.} Illustrated is a comparison of trajectory diversity using the Stein variation ergodic approach and the vanilla ergodic trajectory optimization. (a) Ergodic losses for $20$ trajectories using the proposed Stein variational ergodic method and the vanilla ergodic Trajectories. (b) Trajectory diversity measured using the determinant of the RBF matrix kernel $K_{ij}=k(\bfx^i, \bfx^j)$, (where $K_{ii} = 1$ and $k$ is given by Eq.~\eqref{eq:rbf}) with a fixed $h=0.01$. The initial sampled trajectories are drawn from a zero-mean Gaussian distribution with $\sigma^2=0.01$. Values approaching $1$ suggest diversity of trajectories (the Kernel matrix is strongly diagonal with values $1$). Smaller values suggest similar trajectories (the Kernel matrix consists of identical values).  }
            \label{fig:diversity}
        \end{figure}

    \subsection{Kernel design for Ergodic trajectories} 
            Exploration requires reasoning about long time horizons to promote effective coverage.
            This poses a problem as Stein variational methods tend to perform poorly in high-dimensional problems, e.g., $T=100$ for $n=2$ results in $200$ dimensional inference problem~\cite{zhuo2018mpstein}. 
            The repulsive terms that provide the diversity in solutions vanishes at higher dimensions (due to the derivative of the kernel diminishing with respect to its input). 
            The most common way to effectively deal with the computational issues is by choosing a kernel of the form 
            \begin{equation}\label{eq:markov_kernel}
                k(\bfx, \bfx^\prime ) = \sum_t k(x_t, x_t) + \sum_{(t, t^\prime) \in \mathcal{G}} k(x_t, x_t^\prime)
            \end{equation}
            where $\mathcal{G}$ is a graph of connected points separated by time and $k(\bfx, \bfx^\prime)$ is a sum of positive semi-definite kernels that leverages the Markov property of state trajectories~\cite{lambert2021stein} which we refer to as a Markov Kernel.
            As with prior work, a good choice of kernel is the smooth radial basis function (RBF)
            \begin{equation} \label{eq:rbf}
                k(x, x^\prime) = \exp\left(- \Vert x - x^\prime \Vert_2^2 / h \right)
            \end{equation}
            where $h = \text{med}\left( \{ \bfx \}\right)/ \log N$ is a heuristic that takes the median of the particles. 
            It is possible to choose the kernel to be a composition, e.g., a Markov RBF kernel, which we evaluate the effect of kernels on trajectory diversity in Section~\secref{sec:results}. 

            In this work, we leverage the multiscale aspect of ergodicity and the spectral ergodic metric and define the kernel and trajectories in the workspace $\cW = [0, L_0]\times[0, L_1], \ldots \times [0, L_{v-1}]$ where $L_i$ can be arbitrarily chosen to satisfy numerical conditioning, e.g., $L_i=1$, and $\forall w \in \cW, w\in [0,1]^v$.
            Given an invertible and linear map $g : \cX \to \cW, g^{-1} : \cW \to \cX$, we can rewrite the trajectories in terms of the composition $\bfw = g \circ \bfx = [g(x_0), \ldots, g(x_{T-1})]$ which significantly improves numerical conditioning of the kernel without loss of generality.  
            Note that if $k(x, x^\prime) = 1 \, \forall x, x^\prime \in \Paths(\cX)$ then the Stein variational ergodic optimization problem reduces to a parallel trajectory optimization over random samples $\{ \bfx^i \}_{i=1}^N$. 
            Other choices can provide better options, but in this work we only consider the smooth RBF and the Markov RBF kernel. 

    \subsection{Convergence}

    The convergence of SVGD has been actively studied, primarily in the population limit with infinitely many samples~\cite{liu2017stein, korba2020non-asymptotic, salim2022convergence}. 
    A non-asymptotic analysis of convergence was carried out in~\cite{korba2020non-asymptotic}, and finite sample bounds for propagation of chaos have also been considered~\cite{shi2022finite-particle}. We show in Appendix~\ref{apx:convergence} that SVGD converges in the present setting of ergodic search by applying~\cite[Corollary 6]{korba2020non-asymptotic}.

    \begin{theorem}[Informal]
        Given the assumptions in this section and in the infinite particle limit, SVGD converges, where the gradient update steps are bounded by
        \begin{align}
            \frac{1}{r} \sum_{i=1}^r \|\phi^*_i\|^2_{\cH^{Tv}} \leq \frac{\KL(p||\pi)}{c_\epsilon r},
        \end{align}
        where $p$ is a smooth prior distribution and $c_\epsilon > 0$ is a constant.
    \end{theorem}
    \begin{proof}
        See Appendix~\ref{apx:convergence} for formal proof.
    \end{proof}
    Convergence of SVGD subject to an ergodic metric is beneficial towards guaranteeing trajectory solutions can be acquired. 
    In the following section we provide analysis and empirical results for our proposed approach. 
            
       \begin{figure}
            \centering
            \includegraphics[width=\linewidth]{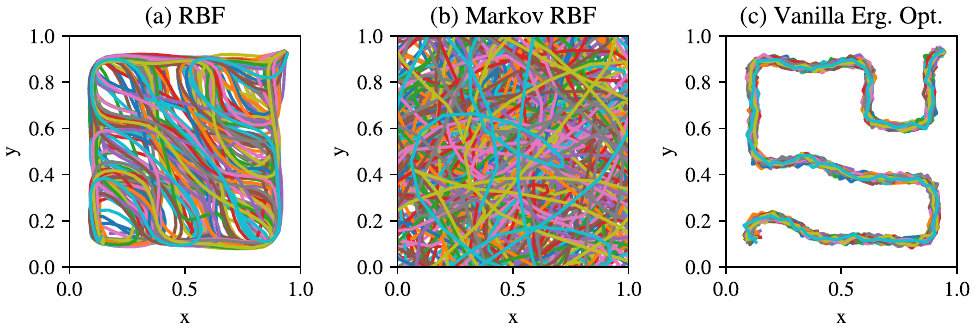}
            \caption{\textbf{Effect of Kernel on Ergodic Trajectory Diversity .} Trajectories are optimized to provide uniform coverage over domain. All $50$ initial trajectory particles are initialized as an interpolating function between initial and final poses additive zero mean noise with $\sigma^2 = 0.01$. Each (a) Radial basis function (RBF) kernels produce smooth diverse paths. (b) Using Eq.~\eqref{eq:markov_kernel} which relies on the Markov property of trajectories $x(t)$, uniform ergodic trajectories can be arbitrarily diverse. (c) Setting the kernel $k(x_i,x_j) = 1$ reduces the Stein ergodic gradient descent to parallel gradient descent on independent trajectory particles $\bfx_i$. Note that particles collapse on a single ergodic path. }
            \label{fig:kernel_anal}
        \end{figure}

        \begin{figure*}[ht!]
            \centering
            \includegraphics[width=\textwidth]{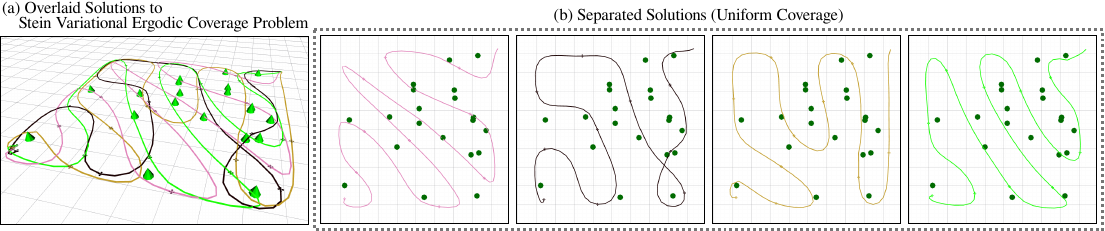}
            \caption{\textbf{The many ways to explore a forest.} Here, we demonstrate the many different solutions to uniform exploration of a $100m \times 100m$ forest that our proposed method produces. (a) Trajectories are optimized through the Stein variational ergodic approach starting from the same initial conditions. (b) Each solution is shown to be a locally optimal solution to the ergodic metric that provides uniform coverage which can be used for robust environmental monitoring.}
            \label{fig:sim_cluttered_search}
        \end{figure*}
        
\section{Results} \label{sec:results}

    In this section, we provide analysis and validation of our proposed approach. 
    Specifically, we are interested in addressing the following:
    \begin{enumerate}
        \item How diverse are Stein variational ergodic trajectories?
        \item Do optimized trajectories yield similar ergodic losses?
        \item How does the choice of kernel effect solutions?
        \item Can the approach handle additional constraints?
        \item Can the proposed approach adapt exploration in a real-time control setting?
    \end{enumerate}
    Additional information regarding parameters, kernels, assumptions, and implementation details are provided in Appendix~\ref{apx:details}. 
    Multimedia for the results can be found in the supplementary material.

    \subsection{Convergence and Trajectory Diversity}
        First, we are interested in evaluating whether the proposed Stein variational ergodic trajectory optimization method converges and provides diverse trajectories.
        In particular, we are interested in empirically evaluating how ergodic trajectories can be under a Stein variational approach.  
        We consider the case of Alg.~\ref{alg:stein_erg_traj_opt} in a 2-D planar uniform coverage problem in a bounded domain (see Fig.~\ref{fig:convergence}).

        The trajectories are constrained such that each consecutive time step is regularized (ensuring smoothness of the paths). 
        In addition, we bound the trajectories within the $\cW = [0,1]^2$ domain through a barrier function (see Appendix~\ref{apx:details}). 
        We evaluate a discrete sample of $N=50$ trajectories of length $T=100$ (initialized with a mean trajectory interpolating from some initial point to a final point on $\cW$ with additive zero mean Gaussian noise with variance $\sigma^2=0.01$. 
        We set $\mu = 1 \, \forall w \in \cW$, i.e., a uniform distribution, to specify the ergodic metric.

        \vspace{0.5em}
        \noindent
        \textbf{Stein Ergodic Trajectory Convergence.}
        As demonstrated in Fig.~\ref{fig:convergence}, the proposed Stein variational ergodic trajectory optimization converges each of the $50$ trajectories towards a minimum ergodic loss (shown in log form).
        Interestingly, the standard deviation of the ergodic losses for the trajectory remains significantly minimal (where we show $2$ standard deviation in Fig.~\ref{fig:convergence}). 
        Visual inspection of the trajectories in Fig.~\ref{fig:convergence} demonstrates the diversity in the optimized solutions which produce uniform coverage over the domain. 
        This result emphasizes that synergy between the spectral ergodic metric and the Stein variational gradient descent, offering the effectiveness of non-parametric optimization in approximating a posterior $p(\bfx | \cO)$ over trajectories, and the ergodic metric producing complex trajectories that facilitates discovery of many exploration strategies. 

       \begin{figure}
            \centering
            \includegraphics[width=\linewidth]{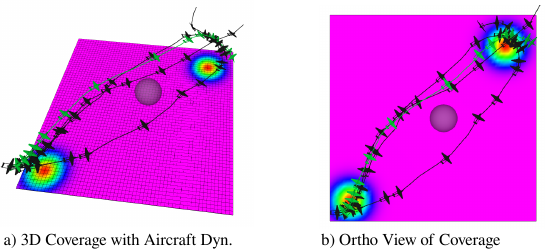}
            \caption{\textbf{3D Diverse Coverage with Aircraft Dynamics.} Here, we show that our proposed Stein ergodic controller (Alg.~\ref{alg:stein_erg_control}) can take into consideration nonlinear aircraft dynamics in a 3D search environment with an obstacle. The underlying distribution represent $\mu$ as a bimodal distribution with two peaks projected on the 2D plane. The spherical obstacle is placed between the two modes for which the controller has to construct $4$ trajectories to explore around the obstacle while spending most of its time around the two peaks. }
            \label{fig:aircraft}
        \end{figure}

       \begin{figure*}[ht!]
            \centering
            \includegraphics[width=\textwidth]{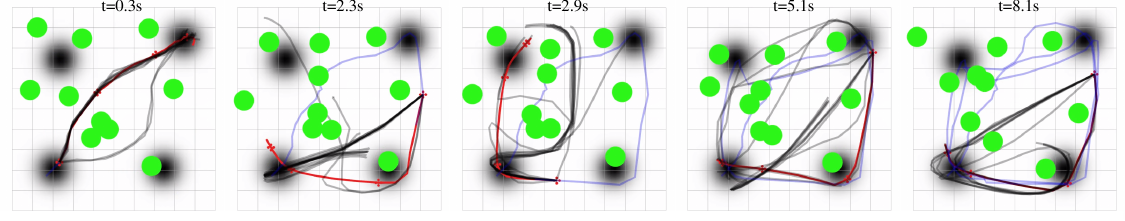}
            \caption{\textbf{Stein Variational Ergodic Model-Predictive Control with Dynamic Obstacles.} Here, we demonstrate the adaptiveness of our proposed method in an online exploration problem with dynamics obstacles. The goal is to visit the dark areas (defining $\mu$) while avoiding the dynamic obstacles (green) which move based on zero mean, $\sigma=0.1$ normally distributed velocity inputs which is unknown to the controller (only instantaneous position of obstacles are known). Trajectory planning (total of $12$ plans are optimized) is calculated using single-integrator dynamics with control constraints so a drone can track the inputs. The red line indicates the best possible solution according to our heuristic in Eq.~\eqref{eq:heuristic}, black lines indicate planned path, and blue line indicates visited areas. The proposed Stein variational ergodic controller optimizes several paths that can be used as redundancies in case plans become infeasible due the dynamic environment.}
            \label{fig:mpc}
        \end{figure*}

       \begin{figure*}[ht!]
            \centering
            \includegraphics[width=\textwidth]{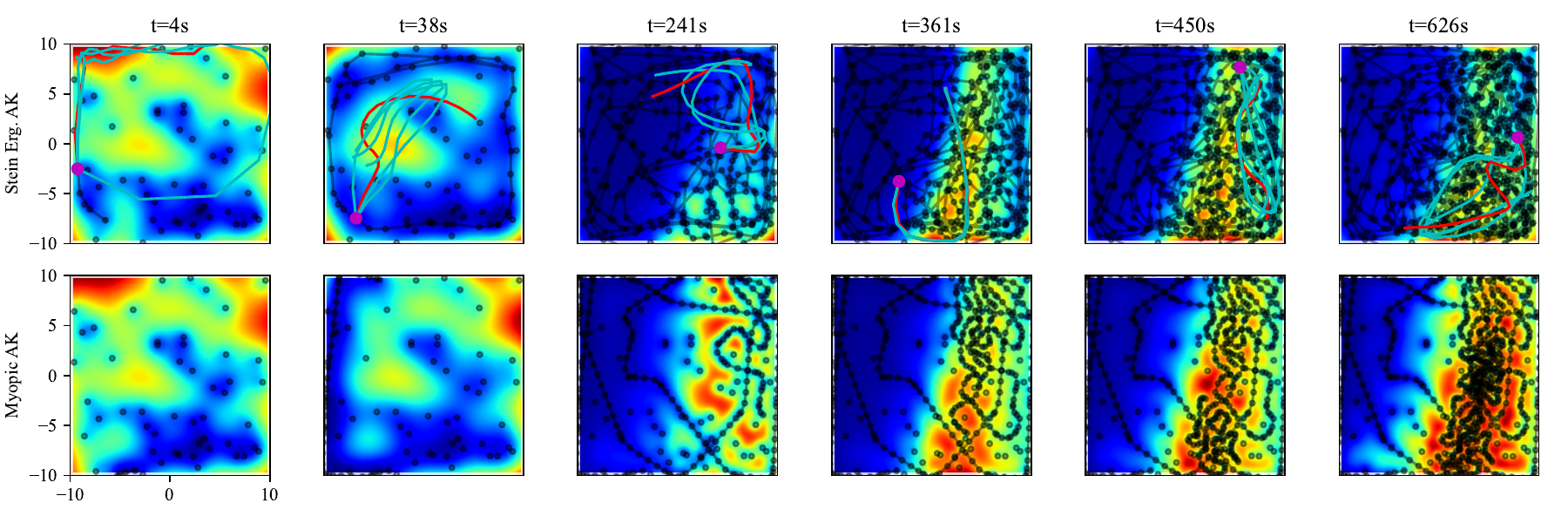}
            \vspace{-2em}
            \caption{\revision{\textbf{Stein Variational Ergodic Sensor-based Exploration with Attentive Kernel.} Here, we demonstrate our approach can readily integrate sensors and information measures is this sonar-based mapping task~\cite{Chen-RSS-22}. The target distribution $\mu$ is specified by the entropy of the attentive kernel (AK) Gaussian process model (developed in~\cite{Chen-RSS-22}) as data is collected (dark circles). A total budget of $700$ points with an initial $50$ provided to initialize model fitting of the underlying terrain (the N17E073 described in~\cite{Chen-RSS-22}). Sensor location plans (total of $4$ plans shown in cyan where the red plan is chosen by a heuristic) are optimized in receding horizon at a $1$ Hz rate for $20s$ into the future. The sensor location is tracked by a Dubins car dynamical system which commits a sensor measurement at the magenta point. We compare against the myopic planner used in~\cite{Chen-RSS-22} where red areas have high entropy and blue areas have low entropy. The Stein variational ergodic approach spreads out the sample points while planning several strategies for maneuvering the sensor location. }}
            \label{fig:sensor_based_exploration}
        \end{figure*}

        \vspace{0.5em}
        \noindent
        \textbf{Stein Kernel Effect on Ergodic Trajectories.}
        Analyzing the kernel gives us more insight as to its effect on producing diverse trajectories. 
        In Fig.~\ref{fig:kernel_anal}, we evaluate the proposed approach in Alg.~\ref{alg:stein_erg_traj_opt} on three kernels, the RBF kernel~\eqref{eq:rbf}, the Markov RBF~\cite{lambert2021stein}~\eqref{eq:markov_kernel}, and a kernel $k(x^i, x^j) = 1$ that reduces the Stein variational gradient to independent parallel gradient descent on a set of trajectories, i.e., a benchmark on the canonical vanilla ergodic trajectory optimization.
        We configure the problem as we did in the convergence analysis presented in Fig.~\ref{fig:convergence} where $\mu$ is a uniform distribution and the trajectories (totalling $50$) are regularized to be Markovian. 
        Each kernel is optimized until convergence of the ergodic metric (based on the gradient condition as described in Alg.~\ref{alg:stein_erg_traj_opt}).  
        We find that the Markov RBF kernel produced significantly more diverse trajectories than the RBF kernel alone.
        This is expected as the Markov RBF kernel repels pair-wise trajectory points where RBF produces compares whole trajectories (which results in more regular trajectories). 
        Compared to the base vanilla ergodic trajectory optimization, the proposed Stein variational ergodic approach significantly outperforms the benchmark, generating a diverse set of ergodic trajectories. 
        Under the same initial conditions, the benchmark ergodic trajectory optimization produces the same ergodic trajectory, resulting in mode-collapse over the posterior and further supporting the need for the proposed Stein variational ergodic method. 

        \vspace{0.5em}
        \noindent
        \textbf{Stein Ergodic Trajectory Diversity.}
        We can see the diversity of the trajectories more clearly in Fig.~\ref{fig:diversity} where we compare the Stein variational ergodic approach with the vanilla ergodic metric. 
        The kernel is given by Eq.~\eqref{eq:rbf} with $h=0.01$. 
        As seen in in Fig.~\ref{fig:diversity}, both methods converge towards an ergodic trajectory local minima. 
        The Stein variational approach produces a wider range of local minima (based on standard deviation) and initially produces worse ergodicity (due to the trajectory diversification). 
        The initial bump in the Stein variational ergodic losses are due to the initial trajectory diversification, as shown in Fig.~\ref{fig:diversity} (b).  
        The diversity measure is calculated using the determinant of the kernel matrix $K_{ij} = k(\bfx^i, \bfx^j)$ which converges on a diverse set of trajectories $\det(K)=1$. 
        Interestingly, this behavior does not influence the rate of change in the ergodic loss as both the Stein variational ergodic approach and the vanilla ergodic trajectory optimization converge onto the same levels of ergodicity within the same number of iterations. 
        The trajectories calculated from parallel ergodic trajectory optimizations (without the repulsive force) collapses onto a single local minima which results in $\det(K)=0$ and non-unique trajectory solutions.

      \begin{figure}
            \centering
            \includegraphics[width=\linewidth]{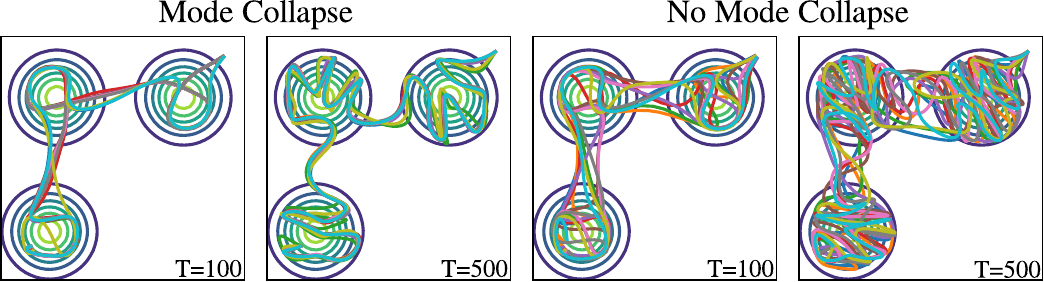}
            \caption{\revision{\textbf{Stein Variational Mode Collapse.} Here, we demonstrate mode collapse of Stein variational inference, i.e., when the repulsive force reduces to zero. (Left) Ergodic trajectories randomly initialized converge to similar solutions when the repulsive force from Stein variational is set to zero. (Right) When the repulsive force is non-zero, the trajectory solutions leverage the multimodal nature of the ergodic metric, yielding diverse solutions. Mode collapse can occur with increased dimensionality of the optimization variables which can be mitigated with the right choice of kernel~\cite{zhuo2018mpstein}. }}
            \label{fig:mode_collapse}
        \end{figure}

    \subsection{Large-Scale and Constrained Exploration} 

        Next, we consider real-world problem scenarios where exploration occurs in a constrained environment (with obstacles and dynamic constraints).
        Our first examples looks at generating diverse exploratory trajectories in a simulated forest. 
        The goal is to generate uniform coverage over a larger domain ($100m \times 100m$) and demonstrate the flexibility of the proposed approach when dealing with arbitrary scales, and introduce several obstacles representing trees.
        The trajectories are constrained to be Markovian (via a penalty term, see Appendix~\ref{apx:details}) to match velocity constraints that are feasible for a drone to follow. 
        The trajectories are solved over seconds to indicate the larger-scale aspect of the example, and are required to start and end at the same point (to further emphasize the diversity of solutions). 
        The trees are represented as disk penalty terms that encompass the width of the tree (and is known at optimization time). 

        \vspace{0.5em}
        \noindent
        \textbf{Multiscale Coverage in a Forest. }
        As illustrated in Fig.~\ref{fig:sim_cluttered_search} is an example of $4$ solutions to the Stein ergodic trajectory optimization in the cluttered forest environment. 
        Note that the initial and final positions are same for each trajectory (due to the cost-terms, see Appendix~\ref{apx:details}), and yet the proposed approach is capable of generating highly diverse solutions (with just $4$ trajectory samples). 
        In addition, the computation of the trajectory incurred no additional cost as the function map $g$ was the only element that was altered (other than the cost terms). 
        As a result, the trajectories were produced with relative ease owing to the synergy between the ergodic metric with Stein variational methods, especially in terms of problem scale. 
        
       \begin{figure}
            \centering
            \includegraphics[width=0.9\linewidth]{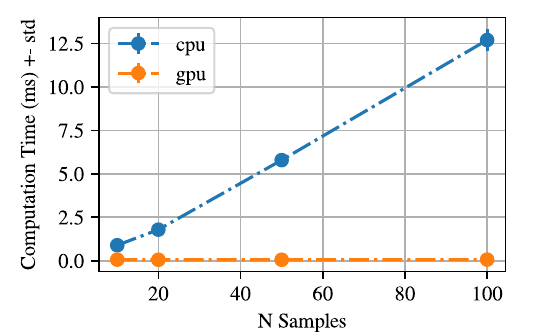}
            \caption{\textbf{Computational Analysis of Stein Variational Ergodic Search.} Here, we show the computational scaling of the Stein variational ergodic search gradient step as a function of the number of particles $N$. Note that with the CPU (AMD Threadripper 3960X), the computational scales linearly. With GPU computation (NVIDIA RTX 3080) we get constant computational scaling (between $57$ and $66 \mu s$). The additional parallelization is afforded by the spectral composition of ergodic search which further improves GPU computation. }
            \label{fig:comp}
        \end{figure}

        \begin{figure}
            \centering
            \includegraphics[width=\linewidth]{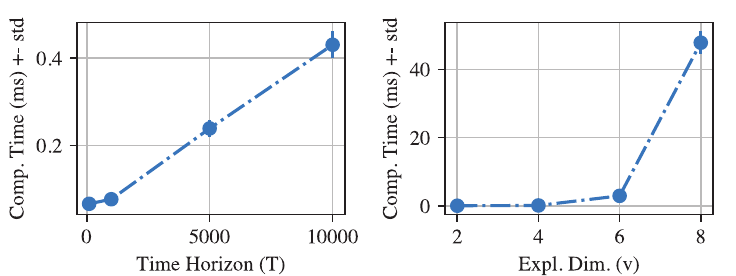}
            \caption{\revision{\textbf{Scale Analysis of Stein Variational Ergodic Search.} Here, we show the computational scaling of the Stein variational ergodic search gradient step as a function the time horizon $T$ and the exploration dimensionality $v$. Note that due to the added parallelization, time horizon scales computation approximately linearly. With increasing exploration dimensionality, computation increases quadratically (which is common to most ergodic exploration methods~\cite{miller2015ergodic, scott2009capturing}). Both examples are done with the same number of basis functions per dimensions.  }}
            \label{fig:comp2}
        \end{figure}
        
       \begin{figure*}[ht!]
            \centering
            \includegraphics[width=\textwidth]{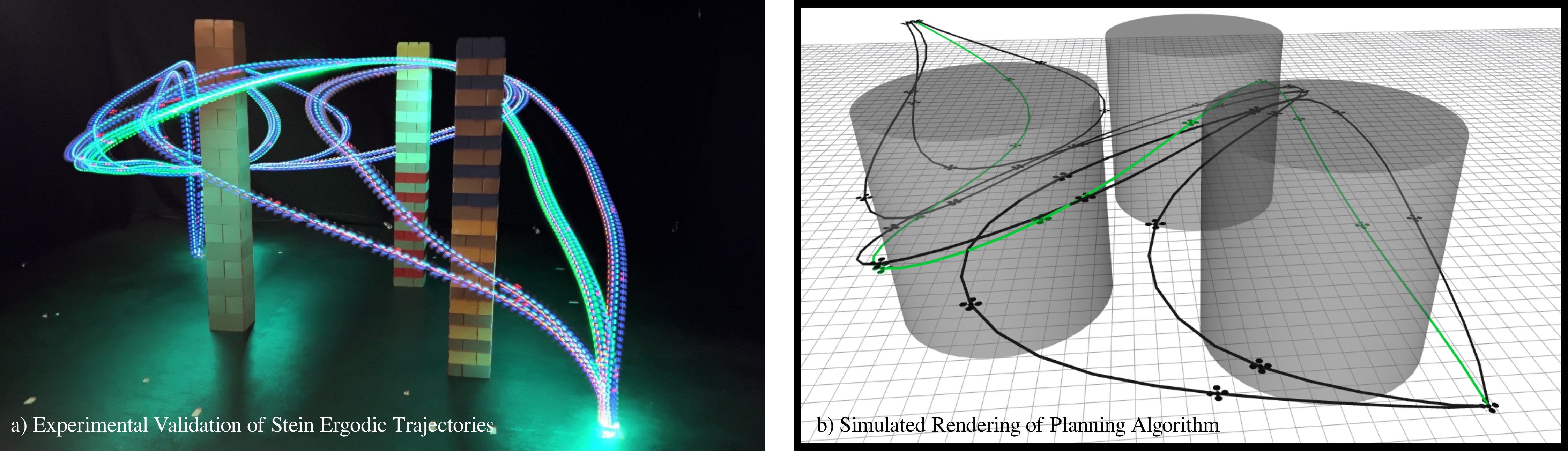}
            \caption{\textbf{3D Coverage Drone Experiment.} Here, we show experimental validation of the proposed Stein variational ergodic search. The trajectories are optimized ($4$ total with $T=50$) to be uniformly ergodic in the domain with three obstacles known to the solver. The green trajectory is the most optimal according to the heuristic~\eqref{eq:heuristic}. (a) Trajectory constraints satisfying the 3D velocity limits are used to produce feasible paths for the Crazyflie 2.1 Drone to follow. (b) Rendering of Stein variational ergodic trajectory optimization through the three obstacles. Note that the trajectories appear symmetric about the first obstacle, but are forced to be diverse through the Stein variational kernels which produce unique trajectories. }
            \label{fig:exp_cluttered_search}
        \end{figure*}
        
        \noindent
        \textbf{Diverse Coverage with Nonlinear Dynamics. } We additionally evaluate our proposed method to generate diverse coverage with nonlinear robot dynamic constraints (see Fig.~\ref{fig:aircraft}). 
        The dynamics of the robot are given by a nonlinear aircraft model (see Appendix~\ref{apx:details}). 
        The trajectory length is given by $T=150$ with $dt=0.1s$ and the Markov RBF kernel~\eqref{eq:markov_kernel} was used in this example. 
        We use the Stein variational ergodic controller (Alg.~\ref{alg:stein_erg_control}) to produce a distribution of diverse exploratory trajectories that avoids the spherical obstacle in the middle. 
        The generated trajectories leverage the nonlinear aircraft dynamics to go around the obstacle and avoid collision while slowing down around the bimodal peaks to generate coverage.

    \subsection{Online Planning and Adaptive Exploration} 

        Last, we explore the capability of the proposed approach for online planning and adaptive exploration through the use of the Stein variational ergodic model-predictive controller (Alg.~\ref{alg:stein_erg_traj_opt}). 
        In particular, we are interested in leveraging the parallel nature of the proposed approach for exploration to efficiently compute exploration strategies in dynamic environments. 
        We evaluate the proposed controller as a receding-horizon planner, which recomputes control strategies based on new information, in a exploration problem with dynamic obstacles. 
        The obstacles are assumed to be observable after they move (and thus only their positions are known to the solver). 
        We plan with a fixed time horizon of $T=20$ steps starting from the robot's current initial condition and plan using $12$ trajectories with an RBF kernel. 
        The dynamics are given by a single-integrator system with velocity limits that match the Crazyflie 2.1 drone~\cite{crazyswarm}. 
        The underlying information measure $\mu$ is given by a mixture of $4$ Gaussians placed asymmetrically over a bounded 2-D domain. 

        \vspace{0.5em}
        \noindent
        \textbf{Online Planning and Control.}
        In Fig.~\ref{fig:mpc} we demonstrate the proposed receding-horizon Stein variational ergodic controller. 
        Illustrated is a time-series snapshot of the controller producing several paths at various stages that avoids obstacles and provides alternative routes to explore the high-density areas. 
        The selected path (in red) is based on the maximum of the cost likelihood model at the given planning time. 
        As a result, it is typical that the approach switches between coverage strategies based on the immediate information obtained form observing the environment. 
        
        Interestingly, we find the proposed control approach is still able to reason about complex trajectories that visit multiple modal peaks within a single path. 
        This is an exceptionally advantageous capability that provide within-trajectory diversity without resorting to heuristics or sample-based augmentations e.g., control randomness~\cite{williams2017model}. 
        This can be seen through the smoothness of the planned trajectory paths (note in this example a trajectory smoothness penalty was not used as the underlying dynamic constraints naturally induce Markov trajectories).
        Furthermore the proposed approach visited each of the modal peaks within a $10s$ run of the controller, even when the obstacles cluttered around one another and cut off certain paths. 
        \revision{
        The proposed controller generates coverage trajectories around the free space proportional to the utility of information in that space. 
        In this scenario, the trajectories will remain in the collision-free space until a path opens, leading to a fast traversal as the ergodic metric will penalize staying in the same area for too long. 
        }

        \vspace{0.5em}
        \noindent
        \revision{
        \textbf{Sensor-based Exploration. } We also demonstrate how our proposed approach can readily incorporate sensors and information measures to guide exploration. 
        A sonar-based mapping problem as described in~\cite{Chen-RSS-22} is used to demonstrate this added capability. 
        The goal is to map an environment using a sonar sensor (see~\cite{Chen-RSS-22} Fig. 7). 
        We adopt the Attentive Kernel (AK) model where we compare against the myopic, sample-based planner that was used for finding where to sample data (as described in~\cite{Chen-RSS-22}). 
        The information measure is defined as the entropy of the Attentive Kernel Gaussian process and is inserted into the Stein variational MPC as the utility measure $\mu$. 
        A total of $1000$ samples are used to compute the Fourier transform of $\mu$ numerically based on the attentive kernel model (as done with the myopic planner). 
        The same experiment settings are maintained where we choose sample locations based on the exploration path determined by the Stein variational ergodic approach (with a budget of 700 data points). 
        Standard Gaussian process metrics are used to compare the quality of the Gaussian process as a function of data like, standardized mean square error (SMSE), mean standardized log loss (MSLL), root-mean-square error (RMSE), mean negative log-likelihood (MNLL), and mean absolute error (MAE) as done in~\cite{Chen-RSS-22}. 

        Empirical results for the resulting Gaussian process metrics are reported in Table~\ref{table:sensor_explr}. 
        The results are averaged over 5 runs of each algorithm subject to the same random seed. 
        Note that the proposed Stein variational ergodic approach improves the each metric without explicit fine-tuning to the sensor. 
        As demonstrated in Fig.~\ref{fig:sensor_based_exploration}, the proposed Stein variational ergodic approach generates several exploration plans that plan useful data acquisition across multiple high-information peaks. 
        In contrast, the myopic planner sets goal points based on a single plan but does not consider the effect collecting the data may have on the overall information landscape. 
        As a result, how data points are distributed is significantly different across methods as our proposed approach spreads out data proportional to the entropy measure rather than densely sampling in only high-entropy areas. 
        It is worth noting that our approach is not specific to the exact sensor or how we define information. The only restriction is that we are able to at least sample from the information distribution to compute the Fourier transform.  
        }

        \begin{table}
        \centering
        \caption{Standard Gaussian Process Metrics for Training Data using Stein Erg. and Myopic Planners (lower is better).}

        \begin{tabular}{*6c}    \toprule
         & SMSE & MSLL & NLPD & RMSE & MAE \\\midrule
         Stein. Erg.  & 0.0669 & -1.6556 & 4.0504 & 16.6373 & 11.6971 \\ 
         Myopic Plan. & 0.0765 & -1.6498 & 4.1785 & 17.7783 & 12.2943\\\bottomrule
         \hline
        \end{tabular}
            \label{table:sensor_explr}
        \end{table}

        \vspace{0.5em}
        \noindent
        \textbf{Computational Analysis. } We additionally analyze the computational time of the proposed method. 
        In Fig.~\ref{fig:comp}, we find that CPU-based computation (using an AMD Threadripper 3960X) and acceleration through JAX~\cite{jax2018github} we get linear computational scaling with respect to the number of samples (holding the time horizon $T=100$ constant). 
        With GPU-based computation and JAX (using an NVIDIA RTX 3080) we get constant scaling (with the worse computation time of $67 \mu s$. 
        The increase in computation is as a result of the spectral nature of the ergodic metric. 
        Specifically, one can additionally distribute the ergodic metric and gradient computation quite effectively which further improves the computation time, allowing for real-time control.

        \revision{
        Furthermore, we provide a brief computational complexity analysis of our methods; in particular for one SVGD step described in Eq.~\eqref{eq:steinE_step}, where we assume the paths $\bfx$ are valued in an exploration space of dimension $v$. This Stein variational step requires the computation of both the kernel matrix $k(\bfx_r^i, \bfx_r^j)$, which has a cost of $\mathcal{O}((NvT)^2)$, and the gradient $\nabla_\bfx \log p(\bfx_r^i)$, which requires the computation of $N$ gradients on an $vT$-dimensional space. We must also compute the spectral ergodic cost function, which is common to all ergodic search methods. For $k$ basis functions in each dimension, the ergodic cost function from Eq.~\eqref{eq:ergodic_met} has a complexity of $\mathcal{O}(NTk^v)$. We emphasize that all of these computations are parallelized in practice, and Fig.~\eqref{fig:comp2} demonstrates the effective parallelized scaling behavior.
        }

    \subsection{Experimental Validation}

        We demonstrate empirical validation of the proposed Stein variational ergodic search method on a 3D uniform coverage problem in a crowded domain (see Fig.~\ref{fig:exp_cluttered_search}). 
        We solve for a set of $4$ ergodic trajectories (using Alg.~\ref{alg:stein_erg_traj_opt}) at a length of $T=50$ with a time step $dt=0.1s$. 
        The coverage problem is defined over the 3D domain $\cW=[0,3] \times [0,3] \times [0.5,1.5]$. 
        Trajectories are constrained to satisfy feasible drone velocities according to the specifications of the Crazyflie 2.1 drone~\cite{crazyswarm}.
        The best trajectory is selected based off the heuristic~\eqref{eq:heuristic}. 
        The path was tracked using a low-level PID controller 
        We find close correspondence with the planned trajectories and what the drone was able to execute in the in a physical experiment. 
        Furthermore, the diversity in coverage solution can be used to readily adapt exploration as needed. 

\revision{
\section{Limitations and Best Practises } 

    \vspace{0.25em}
    \noindent 
    \textbf{Mode Collapse. } A well-known issue with SVGD is the reduction of the repulsive force in Eq.~\eqref{eq:steinE_step} as the dimensionality of the input space increases. This causes the particles to concentrate in the modes of the target distribution potentially diminishing its exploration capabilities and accuracy to capture long-tail distributions. Fortunately, there are several ways to deal with high dimensional inputs by exploiting the structure of our problem. In particular, since we are planning over the space of trajectories, there is an implicit Markovian assumption over the inference variables. This allow us to construct structured kernels such as the sum of kernels defined over subsets of variables that minimize the impact of performing inferences over long sequences~\cite{zhuo2018mpstein}. We demonstrate an example of mode collapse in Fig.~\eqref{fig:mode_collapse} by explicitly setting the repulsive force in Eq.~\eqref{eq:steinE_step} to zero.

    \vspace{0.25em}
    \noindent 
    \textbf{Scaling. } We demonstrate in Fig.~\eqref{fig:comp2} that the Stein variational gradient step empirically scales linearly in the time horizon $T$ due to the added parallelization, despite the $\mathcal{O}(T^2)$ complexity of the kernel matrix evaluation. However, we empirically observe quadratic scaling with respect to the scaling dimension $v$ due to the exponential number of basis functions required in the ergodic cost function~\eqref{eq:ergodic_met}. This scaling behavior is inherently a part of the definition of the spectral cost function and is independent to our primary contribution of incorporating SVGD into ergodic search. This is a well-known limitation in the literature~\cite{miller2015ergodic, scott2009capturing}, and various methods to address this have recently been studied~\cite{abraham2021ergodic, shetty2021ergodic, sun_fast_2024}. Our Stein methodology is already able to generate multiple diverse trajectories within these limitations, and we plan to address these scaling limitations in future work.

    \vspace{0.25em}
    \noindent 
    \textbf{Best Practises. } We found that the less nonlinear the constraints are the easier it was to produce ergodic trajectories. 
    This is true in most trajectory optimization and model-predictive control approaches. 
    However, certain aspects and choices have more influence as a result of the limitations discussed earlier. 
    For example, longer time horizons can be optimized and yield non-myopic search strategies with the appropriate choice of time step $dt$ for the dynamics. 
    Similarly, one can achieve very poor and myopic search strategies when the step-size is chosen too small (which is highly dependent on the time-scale of the dynamics). 
    Prior work on the impact of integration scheme and step-size has been explored~\cite{prabhakar2015symplectic} in isolation and should be referred to for best practises in choosing an appropriate time step and time horizon. 
    
    Another consideration is the number of basis functions used for computing the ergodic spectrum. 
    The more features that the underlying information measure $\mu$ has, the more beneficial it is to have a high number of basis functions (see Fig.~\ref{fig:sensor_based_exploration}). 
    This can incur additional computational costs, especially in high-dimensional exploration tasks. 
    Specialized order reduction techniques, e.g.,~\cite{shetty2021ergodic}, can be used to effectively reduce the required number of basis functions and parallelization can further reduce computational overhead.
    A related consideration is the step-size of the stein variational gradient. 
    We found that a step-size $\epsilon \le 0.5$ was sufficient for most problems and varied depending on the problem complexity and numerical scale, i.e., the more nonlinear the constraints and objective the smaller $\epsilon$ has to be to converge. 
    A line search can be performed to find the optimal step-size for the specific problem.
}   
\section{Conclusion} 
\label{sec:conclusion}

    In summary, we present a novel formulation of coverage and exploration as a variational inference problem on distributions of ergodic trajectories. 
    Efficient computation of posterior ergodic trajectories was facilitated through Stein variational inference that provided sufficient parallelization and approximate inference to make the problem tractable. 
    We found that the Stein variational gradients were well suited to being formulated with ergodic metrics as the natural multiscale aspect of ergodicity produces numerically well-behaved gradients for trajectory optimization and model-based control. 
    As an outcome, we proposed two algorithms: 1) for directly computing distributions of ergodic trajectories; and 2) for calculating robust online model-predictive controller for robotic exploration. 
    
    We demonstrated empirical evidence that showed our proposed approach generated diverse trajectories that can be used for exploration and proved convergence of our approach to optimal ergodic trajectories. 
    Constraints and multiscale implementations were demonstrated as an advantage of our approach through several simulated examples.  
    Furthermore, we show the computational efficiency of our approach through parallel compute which was further improved by the spectral composition of computing ergodic trajectories. 
    Physical experiments on a drone provided additional validation of our approach on real-world systems for robust and adaptive exploration.

\section*{Acknowledgments}
DL was supported by the Hong Kong Innovation and Technology Commission (InnoHK Project CIMDA). CL, IA are supported by Yale University. 

\bibliographystyle{plainnat}
\bibliography{references}

\appendices

\section{Proofs} 

\subsection{Spectral Ergodic Metric}\label{apx:ergodic_cost}
In ergodic theory, ergodicity is typically defined for an entire dynamical systems on a probability space rather than pathwise, which is the approach required here. Indeed, suppose $(\cX, \cF, \mu)$ is a probability space and $f_t: \cX \to \cX$ is an ergodic flow. Then Birkhoff's Ergodic Theorem in continuous time (see~\cite[Section 1.2.2]{krengel2011ergodic}) states that for any $\phi \in \cL^1(\mu)$, time averages of $\phi$ converge to space averages of $\phi$ for $\mu$-almost every initial $x_0 \in \cX$; in other words,
\begin{align}
    \lim_{T \to \infty} \frac{1}{T} \int_0^T \phi(f_t(x_0)) dt = \int_\cX \phi d \mu.
\end{align}

However, for an individual path $x(t) = f_t(x_0)$, there is no guarantee this holds. In fact, the left hand side may not even be well-defined since $\phi \in \cL^1(\mu)$ is only defined $\mu$-almost everywhere and the time averaged measure $\cQ_T$ from Eq.~\eqref{eq:Q_T} may not be absolutely continuous with respect to $\mu$. Thus, we use an alternative definition, replacing the $\cL^1(\mu)$ functions with continuous functions $\cC(\cX)$ for the pathwise Definition~\ref{def:ergodicity}.

This pathwise definition generalizes the spectral cost function, which has previously been used as a measure of pathwise ergodicity~\cite{mathew2011metrics, Dong-RSS-23}.

\begin{theorem}
    Let $\cW = [0, L_0] \times \ldots \times [0, L_{v-1}] \subset \R^v$. The spectral cost function is an ergodic cost function.
\end{theorem}
\begin{proof}
    Here, we view $\cW$ with periodic boundaries, and is thus a $v$-dimensional torus (ie. $v$ copies of the circle $S^1$). Suppose $\cE_\mu$ is the spectral cost function defined in Eq.~\eqref{eq:ergodic_met}, and let $x(t): \R^+ \to \cX$ be a trajectory such that $\cE_\mu(x|_{[0,T]}) \to 0$ as $T \to \infty$. This implies that the Fourier coefficients converge $\cQ_T^k \to \mu^k$ as $T \to \infty$, and thus the Fourier transforms of $\cQ_T$ converge to the Fourier transform of $\mu$. Finally, by Levy's convergence theorem for the torus, $\cQ_T$ converges weakly to $\mu$ as $T \to \infty$.
\end{proof}

\subsection{Convergence of SVGD} \label{apx:convergence}
In this appendix, we show that the general convergence results for SVGD from~\cite{korba2020non-asymptotic} hold in the ergodic search setting. We consider discretized ergodic search on a bounded state space $\cX \subset \R^n$ and normalized workspace $\cW = [0,1]^v$ with $T$ time points. Without loss of generality we set $\cX = [0,1]^n$ and thus, we consider SVGD on the space $\cX^T = [0,1]^{Tn}$. We fix a function $g: \cX \to \cX$. Moreover, we work in the population limit of infinitely many initial samples from the prior distribution.

Let $\cE_\mu$ be the spectral ergodic cost function with respect to a measure $\mu$ on $\cW$ and $p$ be the prior distribution on $\cX^T$. Then, the target distribution we aim to approximate is the posterior $\pi \coloneqq p(\bfx | \cO) \sim p(\cO | \bfx) p(\bfx)$ which has the form 
\begin{align}
    \pi \sim \exp\left( - \lambda \cE_\mu (\bfx) + \log(p(\bfx)) \right).
\end{align}
The function
\begin{align} \label{eq:potential}
    V(\bfx) = \lambda \cE_\mu (\bfx) - \log(p(\bfx))
\end{align}
is called the \emph{potential function} in the SVGD literature. Furthermore, suppose $k: \cX^T \times \cX^T \to \R$ is the RBF kernel and $\cH$ is its corresponding RKHS.

Convergence of SVGD is framed in terms of the kernel Stein discrepancy (KSD). In particular, the KSD of a measure $q$ with respect to the target measure $\pi$ is
\begin{align}
    \KSD_\pi(q) \coloneqq \| \phi^*(q)\|^2_{\cH^{Tn}},
\end{align}
where $\phi^*_q \in \cH^{Tn}$ (note that $\cX^T \subset \R^{Tn}$),
\begin{multline} \label{eq:phi_star_inf}
    \phi^*(q) 
    = \E_{\bfx \sim q} [ k(\bfx,\cdot) (\nabla_\bfx \log p(\bfx) - \\ \lambda \nabla_\bfx \cE_\mu(\bfx)) + \nabla_\bfx k(\bfx,\cdot)].
\end{multline}
In particular, $\KSD_\pi(q)$ is determined by the norm of the gradient in the SVGD update step (compare with the finite particle gradient $\phi^*_r$ in Eq.~\eqref{eq:phi_star}. Here, we define SVGD in the population limit in the same manner as the finite particle setting. The population gradient at step $r$ is $\phi^*_r \coloneqq \phi^*(q_r)$, where the step $r$ measure $q_{r}$ on $\cX^{T}$ is defined as the pushforward of $q_{r-1}$ along the function $U_r : \cX^T \to \cX^T $, 
\begin{align}
    U_r(\bfx) \coloneqq \bfx + \epsilon \phi^*_{r-1}(\bfx)
\end{align}
with step size $\epsilon > 0$ and initial condition $q_0 = p$. 


Furthermore, the SVGD convergence results from~\cite{korba2020non-asymptotic} rests on three assumptions on the kernel $k$, the potential function $V$, and moments of the step $r$ measures $q_r$. In particular, there exist constants $B, C, M > 0$ such that the following holds.
\begin{itemize}
    \item[($A_1$)] $\|k(\bfx,\cdot)\|_\cH, \|\nabla_\bfx k(\bfx, \cdot)\|_{\cH^{Tn}} \leq B$.
    \item[($A_2$)] The Hessian $H_V$ of the potential function $V$ from Eq.~\eqref{eq:potential} is well-defined as $\|H_V\|_{op} \leq M$.
    \item[($A_3$)] For all $r$, $\KSD_\pi(q_r) < C$.
\end{itemize}

\setcounter{theorem}{0}
\begin{theorem} (Formal). Let $\cX = [0,1]^n$, $\pi = p(\bfx | \cO)$ be the target distribution with potential $V$ from Eq.~\eqref{eq:potential}, $p$ be smooth prior distribution on $\cX^T$, $k$ be the RBF kernel on $\cX^T$ with RKHS $\cH$. There exists a step size $\epsilon < S$, where $S$ is a constant which depends on $B,C,M$ such that
    \begin{align}
        \frac{1}{r} \sum_{i=1}^r \KSD_\pi(q_i) \leq \frac{\KL(p|\pi)}{c_\epsilon r},
    \end{align}
    where $c_\epsilon$ is a constant which depends on $\epsilon, M, B$.
\end{theorem}
\begin{proof}
    This result is a special case of~\cite[Corollary 6]{korba2020non-asymptotic}, and in order to prove this result, we must show that assumptions $(A_1), (A_2), (A_3)$ above hold. First, $(A_1)$ holds since the RBF kernel is differentiable, and $\cX^T$ is a bounded domain. Next, we note that for a discrete path $\bfx = [x_0, \ldots, x_{T-1}]$, the spectral ergodic cost fucntion has the form
    \begin{align}
        \cE_\mu(\bfx) = \sum_{k \in \mathcal{K}^v} \Lambda_k \left( \frac{1}{T} \sum_{t=0}^{T-1} F_k(g(x_t)) - \int_{\cW} F_k(w) d\mu(w)\right)^2,
    \end{align}
    which is smooth since $F_k$ from Eq.~\eqref{eq:Fk} is smooth. Because the prior $p$ is also smooth, the potential $V$ is smooth, and the Hessian is well-defined. Furthermore, the Hessian is bounded since $\cX^T$ is a bounded domain, so $(A_2)$ is satisfied. Finally, since $(A_1)$ and $(A_2)$ is satisfied, it suffices to show that
    \begin{align}
        \sup_{r} \int_{\cX^T} \|\bfx\| \, dq_r(\bfx) < \infty,
    \end{align}
    to show $(A_3)$, from the discussion in~\cite[Section 5]{korba2020non-asymptotic}. However, this is satisfied since $q_r$ is a probability measure and $\cW^T$ is a bounded domain, so $(A_3)$ is satisfied. 
\end{proof}

\section{Implementation Detail \& Additional Results} \label{apx:details}

    Here, we provide additional implementation details for the experimental results in Section~\secref{sec:results}. For all methods, the Stein variational gradient step size is given by $\epsilon=0.5$ unless otherwise specified. In addition, the RBF kernel is used in all examples unless otherwise specified. Convergence condition is the same for all examples $\Vert \phi_r \Vert \le 10^{-3}$ unless specified.

    \vspace{0.25em}
    \noindent
    \textbf{Convergence and Diversity Results.} For all examples we assume a bounded 2D domain of size $\cW = [0,1]\times [0,1]$. 
    All trajectories are of length $T=100$ with $n=2$ and an assumed $dt=0.1s$. 
    The solvers are all optimized until the gradient direction $\Vert \phi_r \Vert \le 10^{-3}$. 
    The cost function $\cL_\mu$ is given as 
    \begin{align*}
        \cL_\mu(\bfx) = \cE_\mu(\bfx) + 0.1 c_\cW(\bfx) + \sum_t 15 \Vert x_{t+1} - x_{t} \Vert_2^2 \\ 
        + 0.1 \Vert x_0 - x_\text{init}\Vert_2^2  + 0.1 \Vert x_{T} - x_\text{final}\Vert_2^2
    \end{align*}
    where $c_\cW$ is a cost on staying within the boundaries of $\cW$ ($0$ within the boundaries and proportional square of the distance outside the boundaries), and $x_\text{init}$ and $x_\text{final}$ are the initial and final state conditions.
    All initial samples $\{ \bfx^i \}_{i=1}^N$ are drawn from $p(\bfx) = \mathcal{N}(\hat{\bfx}, \sigma^2)$ where $\sigma^2 = 0.01$, and $\hat{\bfx} = \text{interp}(x_\text{init}, x_\text{final})$. A temperature value $\lambda=10$ was used for all experiments. The maximum number of basis functions used was $k_\text{max} =8$ for all dimensions.

    \vspace{0.25em}
    \noindent
    \textbf{Multiscale Constrained Exploration.}
        This experiment was done using Alg.~\ref{alg:stein_erg_traj_opt} subject to a uniform $\mu$ in 2D. 
        The function $g(x)$ would map the trajectory within the $100m\times 100m$ domain onto a $[0,1] \times [0,1]$ domain for improved numerical conditioning. 
        Note that this did not change the underlying problem as the ergodic metric is defined over the Fourier spectral domain which can be defined on any periodic domain.
        The trajectories are of length $T=100$ with a time step of $dt=1s$ and a dimension of $n=2$. A maximum of $k_\text{max}=8$ basis functions were used. 
        The cost function is given as 
        \begin{align*}
            \cL_\mu(\bfx) = \cE_\mu(\bfx) + 0.1 c_\cW(\bfx) + \sum_t 15 \Vert x_{t+1} - x_{t} \Vert_2^2 \\ 
            + 0.1 \Vert x_0 - x_\text{init}\Vert_2^2  + 0.1 \Vert x_{T} - x_\text{final}\Vert_2^2 + 0.01 c_\text{obs}(\bfx)
        \end{align*}
        where $c_\text{obs}(\bfx)$ is a obstacle constraint given by $\text{max}(0,\Vert x - x_c \Vert_2 - r)$ where $r$ is the radius of the obstacle. 
        The initial and final positions of the trajectory are specified by the boundaries of the domain.

    \vspace{0.25em}
    \noindent
    \textbf{3D Constrained Exploration with Nonlinear Dynamics.} 
    In this example the state of the aircraft dynamics are given as $x = [p_x, p_y, p_z, \psi, \phi, v]^\top$ where $p_i$ is the position in 3D, $\phi, \phi$ are the roll and pitch, and $v$ is the forward velocity of the aircraft. 
    The control vector is given by $u = [u_1, u_2, u_3]^\top$ where the dynamics is defined as 
    \begin{equation}
        \frac{d}{dt}\begin{bmatrix}
            p_x \\ p_y \\ p_z \\ \psi \\  \phi \\ v
        \end{bmatrix} = \begin{bmatrix}
            v \cos(\phi) \cos(\psi) \\ 
            v \cos(\phi) \sin(\psi) \\ 
            v \sin(\phi) \\ 
            u_1 \\ 
            u_2 \\ 
            u_3
        \end{bmatrix}.
    \end{equation}
    The distribution $\mu$ is given by a sum of two Gaussians with equal weights. 
    The space is defined by $\cW = [0,3] \times [0,3] \times [0,3]$ and $g(x)$ takes the states and returns the position vector of the aircraft. 
    The obstacle is represented by a distance penalty function $\text{max}(0,\Vert x - x_c \Vert_2 - r)$ where $r$ is the radius of the obstacle. 
    Trajectories are of length $T=150$. 
    Initial controls are sampled from $p(\bfu) = \mathcal{N}(0,\sigma^2)$ where $\sigma^2=0.1$. 
    The cost function is given as 
    \begin{align*}
        \cL_\mu(\bfu) = \cE_\mu(\bfu) + c_\cW(\bfx) + \sum_t \Vert p_{z,t} \Vert_2^2 +  100 c_\text{obs}(\bfx)
    \end{align*}
    and $\bfx$ is computed using $x_{t+1} = F(x_t, u_t)$ from the initial condition. The temperature $\lambda$ has value $20$ and $N=5$ samples are computed. The step size of the Stein variational algorithm is given as $\epsilon = 0.1$.

    \vspace{0.25em}
    \noindent
    \textbf{Stein Variational Model-Predictive Control.} 
        In this example, we consider coverage over a quad-modal distribution $\mu$ which is composed of $4$ Gaussians with equal weights. The domain is specified over $\cW = [0,1] \times [0,1]$ and the dynamics are given as a single integrator $x_{t+1} = x_t + dt u_t$. 
        The planning horizon is given as $T=20$ with $N=20$ trajectory samples. 
        We run each planning loop for a total of $100$ iterations or unless tolerance was achieved of $\gamma = 10^{-2}$. 
        A total of $k_\text{max} = 10$ basis functions are used for each dimension. 
        The $10$ obstacles are randomly placed using a uniform distribution on $\cW$ and move according to random velocity directions sampled from a normal distribution with standard deviation of $\sigma=0.01$ at the same control rate as the robot ($dt=0.1$). 
        The cost function is given as 
        \begin{align*}
            \cL_\mu(\bfu) &= \cE_\mu(\bfu) + c_\cW(\bfx) + 0.01\Vert \bfu \Vert_2^2 \\
            &+   \sum_t 0.001 \Vert x_{t+1} - x_{t} \Vert_2^2 
            + 100 c_\text{obs}(\bfx)
        \end{align*}
        with $\lambda=10$. The initial control prior distribution is given as $p(\bfu) = \mathcal{N}(0, \sigma^2=0.01)$.

    \vspace{0.25em}
    \noindent
    \revision{
    \textbf{Sensor-based Exploration.} 
        In this example, we consider coverage over the entropy of a Gaussian process using the AK attentive kernel as $\mu$. The domain is specified over $\cW = [-10,10] \times [-10,10]$ and the sensor planning dynamics are given as a single integrator $x_{t+1} = x_t + dt u_t$. 
        The trajectory is tracked using a Dubins car dynamics model with a feedback controller rate of $10$ Hz as done in~\cite{Chen-RSS-22}.
        The planning horizon is given as $T=20$ with $N=4$ trajectory samples. 
        We run each planning loop for a total of $100$ iterations or unless tolerance was achieved of $\gamma = 10^{-2}$. 
        A total of $k_\text{max} = 20$ basis functions are used for each dimension. 
        Samples are committed using a sonar sensor model at each $1$ Hz rate location. 
        The cost function is given as 
        \begin{align*}
            \cL_\mu(\bfu) &= \cE_\mu(\bfu) + 0.1 c_\cW(\bfx) + 0.001\Vert \bfu \Vert_2^2
        \end{align*}
        with $\lambda=100$ and $p(\bfu) = \mathcal{N}(0, \sigma^2=0.01)$.
        }

    \vspace{0.25em}
    \noindent
    \textbf{3D Coverage using the Crazyflie Drone.} 
        This example was done on a 3D domain $\cW = [0,3] \times [0,3] \times [0.5,1.5]$ with $\mu$ as a uniform distribution. 
        Trajectories were calculated using Alg.~\ref{alg:stein_erg_traj_opt} with cost functional similar to the multiscale forest example
        \begin{align*}
            \cL_\mu(\bfx) = \cE_\mu(\bfx) + 0.1 c_\cW(\bfx) + \sum_t 15 \Vert x_{t+1} - x_{t} \Vert_2^2 \\ 
            + 0.1 \Vert x_0 - x_\text{init}\Vert_2^2  + 0.1 \Vert x_{T} - x_\text{final}\Vert_2^2 + 0.01 c_\text{obs}(\bfx).
        \end{align*}
        The trajectory length was given by $T=50$ with $dt=0.1s$. 
        The total number of basis functions was given as $k_\text{max}=8$ for each dimension (totalling $512$ basis functions). 

        Once trajectories were computed, tracking was done through the Crazyflie 2.1 API which sends way-point commands as a specified interval (equal to the trajectory time step). 
        Drone position tracking was done using two Lighthouse VR trackers distributed across the domain. 
        The obstacles are defined using an obstacle penalty function given by $\text{max}(0,\Vert x - x_c \Vert_2 - r)$ where $r$ is the radius of the obstacle. 
        The planner was given the positions of the obstacles at planning time. 
        Trajectories are optimized until convergence. 
        Once trajectories are solved, each is send to the drone to execute. The one that satisfies the heuristic in Eq.~\eqref{eq:heuristic} signals the drone to turn a LED green to indicate the optimal solution.

\end{document}